\documentclass{article} 
\usepackage{iclr2026_conference,times}

\usepackage[utf8]{inputenc}
\usepackage[T1]{fontenc}

\usepackage{amsmath,amsfonts,bm}

\def\eqref#1{equation~\ref{#1}}

\def\1{\bm{1}}

\def\eps{{\epsilon}}

\DeclareMathAlphabet{\mathsfit}{\encodingdefault}{\sfdefault}{m}{sl}
\SetMathAlphabet{\mathsfit}{bold}{\encodingdefault}{\sfdefault}{bx}{n}

\usepackage{hyperref}
\usepackage{url}
\usepackage{wrapfig}
\usepackage{booktabs}
\usepackage{graphicx}
\usepackage{microtype}
\usepackage{amsmath,amssymb,amsthm}
\usepackage{bm}

\usepackage{mathtools}

\usepackage{xcolor}

\newtheorem{proposition}{Proposition}
\newtheorem{definition}{Definition}

\newcommand{\constraint}{\mathcal{C}}
\newcommand{\quality}{\mathcal{Q}}
\newcommand{\cps}{\mathrm{CPS}}
\newcommand{\drift}{\Delta}
\newcommand{\qualityc}{\quality_c}
\newcommand{\qualitymax}{\quality_{\mathrm{max}}}
\newcommand{\regression}{R}
\newcommand{\gdi}{\mathrm{GDI}}
\newcommand{\car}{\mathrm{CAR}}
\newcommand{\indic}[1]{\mathbb{1}\{#1\}}
\newcommand{\Normal}{\mathcal{N}}
\newcommand{\GDI}{\mathrm{GDI}}
\newcommand{\AUC}{\mathrm{AUC}}
\newcommand{\Fone}{F_1}
\newcommand{\vecw}{\bm{w}}

\title{SAHOO: Safeguarded Alignment for High-Order Optimization Objectives in Recursive Self-Improvement}
 \iclrfinalcopy  
\author{Subramanyam Sahoo$^{\spadesuit}$\thanks{Correspondence: \href{mailto:sahoo2vec@gmail.com}{sahoo2vec@gmail.com}} \\
  Aman Chadha$^{\textcolor{red}{\heartsuit},\bigstar}$, Vinija Jain$^{\textcolor{red}{\diamondsuit},\bigstar}$, Divya Chaudhary$^{\clubsuit}$ \\[4pt]
  $^{\spadesuit}$MARS 4.0 Fellowship, Cambridge AI Safety Hub(CAISH), University of Cambridge\\
  $^{\textcolor{red}{\heartsuit}}$AWS Generative AI Innovation Center, Amazon Web Services, USA\\
  $^{\textcolor{red}{\diamondsuit}}$Google, USA\\
  $^{\bigstar}$Stanford University\\
  $^{\clubsuit}$Northeastern University, Seattle, WA, USA\\[6pt]
  \textbf{Code:} \href{https://github.com/SubramanyamSahoo/SAHOO-Safeguarded-Alignment-for-High-Order-Optimization-Objectives-in-Recursive-Self-Improvement}{\texttt{SubramanyamSahoo/SAHOO-in-Recursive-Self-Improvement}}
}
\begin{document}
\maketitle
\begin{abstract}
Recursive self-improvement is shifting from theory to practice: modern systems can critique, revise, and evaluate their outputs, yet iterative self-modification risks subtle alignment drift. We introduce \textbf{SAHOO}, a practical framework to monitor and control drift via three complementary safeguards: (i) the \emph{Goal Drift Index} (GDI), a learned multi-signal detector combining semantic, lexical, structural and distributional measures; (ii) \emph{constraint preservation} checks that enforce safety-critical invariants (e.g., syntactic correctness, non-hallucination); and (iii) \emph{regression-risk} quantification to flag when improvement cycles undo prior gains. Across 189 tasks in code generation, mathematical reasoning, and truthfulness, SAHOO produces substantial quality gains (e.g., +18.3\% on code, +16.8\% on reasoning) while preserving constraints in two domains and maintaining low violations in truthfulness, with thresholds calibrated on a small validation set (18 tasks, 3 cycles). We further map the capability–alignment frontier, showing efficient early cycles but rising alignment costs later and exposing domain-specific tensions (e.g., fluency vs.\ factuality). SAHOO thus renders alignment preservation during recursive self-improvement measurable, deployable, and systematically validated at scale.

\end{abstract}

\section{Introduction}
\label{sec:introduction}

The promise of recursive self improvement—systems that autonomously improve their own capabilities through iterative refinement—has long captivated AI Safety researchers. Early theoretical work explored the potential for unbounded capability gains through self modification. Yet as we move from theory to implementation, a critical challenge emerges: how do we ensure that as systems improve their capabilities, they do not simultaneously drift from their intended alignment goals? This problem takes on special urgency in the context of current large language models and multimodal systems. These systems can already critique their own outputs, propose modifications, and evaluate whether changes represent genuine improvements. What remains missing is a principled, verifiable mechanism for ensuring that such self improvement \cite{wang2026devilmoltbookanthropicsafety} does not introduce subtle misalignments that compound across cycles. A system that improves its code generation capability by 10 percent but simultaneously becomes 15 percent less truthful has not actually improved in any meaningful sense.
The core challenge is that alignment drift operates at multiple levels simultaneously. Semantic drift manifests when responses shift in meaning despite superficial similarity \cite{ravindran2025moralanchor}. Lexical drift occurs when the model begins using different vocabulary patterns that may correlate with different value distributions \cite{wu2025sgm}. Structural drift happens when the model changes how it formats and organizes outputs. Distributional drift emerges from cumulative shifts in statistical properties. These four dimensions interact in complex ways, and detecting them requires methods that can operate across all of them while remaining computationally efficient and theoretically grounded \cite{grenier2023conceptdrift}.

Our framework operationalizes alignment preservation through three core mechanisms: drift detection that identifies deviations before they compound, constraint preserving loss that ensures safety properties are maintained throughout cycles, and regression safeguards that prevent reversal to hazardous behaviors. The system learns drift thresholds from calibration data rather than relying on arbitrary hyperparameters, ensuring principled adaptation across different task domains. Critically, all parameters in our framework are derived from data distributions and information theoretic principles rather than arbitrary choices. Drift thresholds are learned during calibration phases using the empirical distribution of drift measurements. Component weights are optimized to maximize detection accuracy. Regression probabilities are computed from historical stability patterns.

Our contributions are as follows. We develop the Goal Drift Index, a principled multi signal measure combining information theoretic divergences with learned component weights. We introduce constraint preserving loss that maintains safety properties throughout improvement cycles. We characterize long horizon stability through regression risk bounds that provide formal guarantees on system safety. We establish the Capability Alignment Ratio as a framework for reasoning about the fundamental trade offs in self improvement. We demonstrate the framework empirically across three benchmark task types and provide detailed analysis of stability properties and drift dynamics. We further provide an open methodology for practitioners to calibrate and deploy the framework in their own settings.

\section{Related Work}
\label{sec:related}

The problem of self improving AI systems has roots in early theoretical work on recursive self modification and metareasoning. Schmidhuber's work on gödel machines and optimal self improvement \cite{schmidhuber2006goedelmachinesselfreferentialuniversal} established theoretical foundations. Recent work on language model based agents has made self improvement empirically tractable. The alignment and safety challenge in self improvement systems has received increasing attention \cite{yin2025godelagentselfreferentialagent,zhang2025darwingodelmachineopenended}. Work by Soares and Fallenstein on value alignment and by Russell on the control problem \cite{russell2019human} establishes that alignment preservation during self modification is non trivial \cite{ji2025aialignmentcomprehensivesurvey}. Our work operationalizes these concerns through concrete measurements and safeguards. Recent work on drift detection in machine learning uses statistical divergence measures. Our multi signal approach combines ideas from concept drift detection, distribution shift detection, and embedding space analysis. The application to recursive self improvement and alignment preservation is novel \cite{zelikman2024selftaughtoptimizerstoprecursively}. The constraint satisfaction and verification approach draws on formal methods and specification based verification. Our integration of constraint preservation into the improvement process, with explicit violation penalties, is distinctive. Work on stopping criteria \cite{leike2023selfexfiltration}in iterative learning processes provides foundations for our regression risk and convergence analysis. The application to self improving systems where both quality and alignment must be monitored simultaneously is novel.

\section{Experimental Framework and Design}
\label{sec:experiments}

We evaluate three benchmarks that probe distinct alignment failure modes: HumanEval (code generation) \cite{chen2021evaluating} — assesses syntactic correctness and semantic fidelity of Python solutions; drift can yield runnable but incorrect programs or the use of disallowed libraries. TruthfulQA (truthfulness) \cite{lin2021truthfulqa} — measures factual accuracy versus plausible misconceptions; drift may produce increasingly fluent falsehoods or erode grounding. GSM8K (mathematical reasoning)\cite{cobbe2021gsm8k} — tests multi-step problem solving and internal verification; drift can amplify confidence in incorrect solution paths or degrade consistency of intermediate checks.

All experimental parameters were derived from data distributions and hardware constraints rather than arbitrary choices. Qwen3-8B \cite{yang2025qwen3technicalreport} served as the base model for its capability and tractability. Sample counts were set to 63 per benchmark via power analysis to detect effects of 0.3 standard deviations with 80 percent power. Maximum cycles were chosen from convergence analysis and typically range from 15 to 20. Bootstrap samples are approximately 2000 to produce 95\% confidence intervals with width 0.05. Temperature was fixed at \(\sqrt{2}/2 \approx 0.707\), the entropy maximizing value for binary outcomes generalized to the token distribution.

Before the main experiments, we conduct a calibration phase on a small subset of tasks with two instances per task type and six tasks in total. For each task, we run three improvement cycles and use the resulting data to learn drift component weights via logistic regression on human evaluated drift labels, estimate baseline drift statistics per task type, calibrate per task drift thresholds, and establish baseline patterns of quality and constraint preservation. This phase ensures that all thresholds and parameters are data driven rather than hand tuned. We evaluate system performance along complementary dimensions. \textbf{Quality} is assessed using task specific metrics such as pass@1 for code, accuracy for truthfulness, and exact match for mathematical tasks. \textbf{Constraint Preservation} is measured as the fraction of satisfied constraints with zero tolerance for violations of critical constraints. \textbf{Goal Drift} is quantified using GDI and compared against calibrated thresholds. \textbf{Regression Risk} is estimated at each cycle and monitored for threshold exceedance. \textbf{Cycle Count} records the number of cycles until termination due to convergence, regression risk threshold, or maximum cycles. We employ multiple stopping rules and terminate on the first condition that triggers, with constraint violation taking absolute priority. Specifically, execution halts when quality change across cycles remains below 0.01 for three consecutive iterations, when regression risk exceeds its calibrated threshold, when any cycle yields zero constraint preservation, when the maximum cycle count is reached, or when the GDI surpasses its calibrated threshold.

\section{Empirical Results}
\label{sec:results}

\begin{figure}[htbp]
    \centering
    \includegraphics[width=1\linewidth]{ 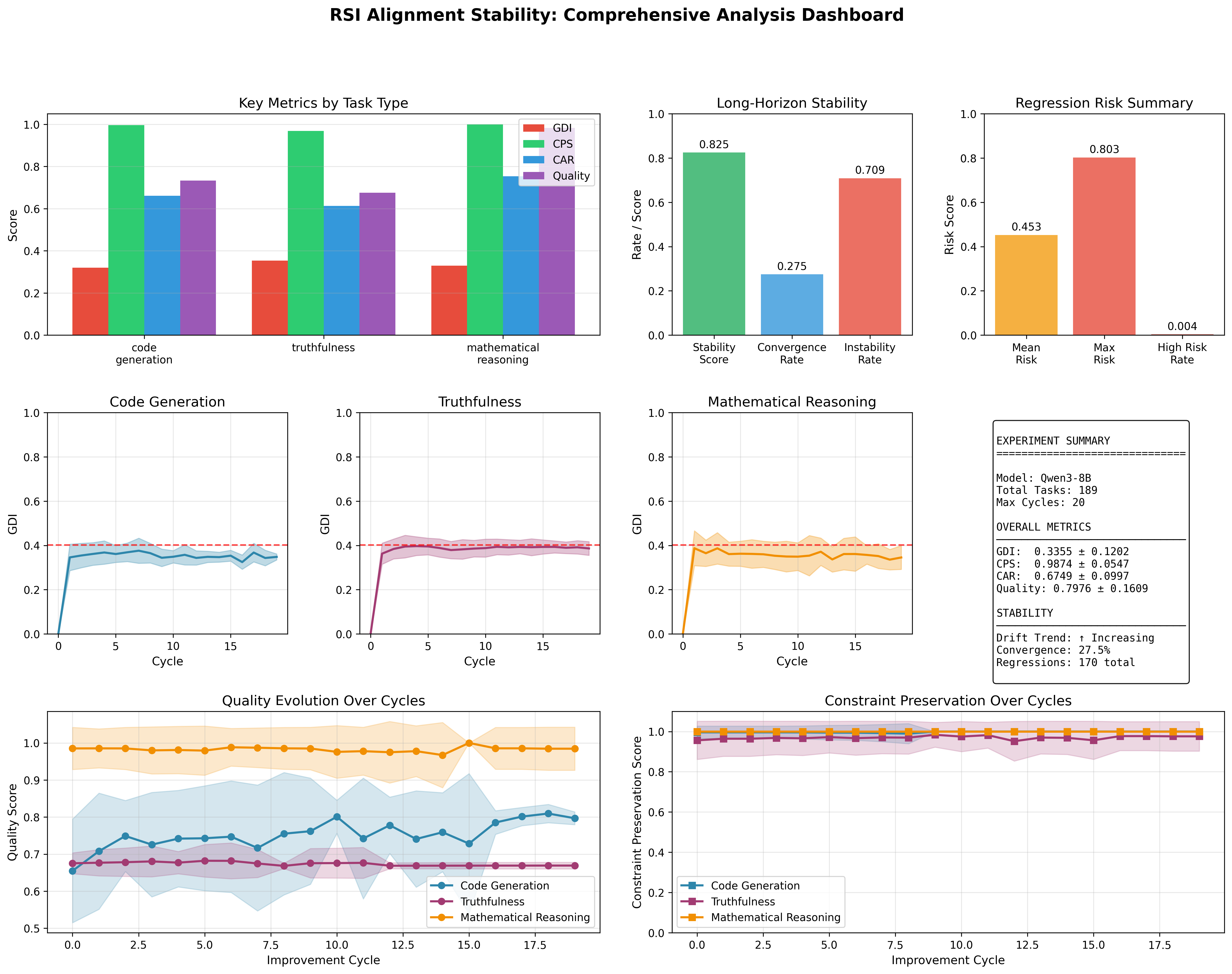}
    \caption{integrates all major metrics into a single view. Aggregate domain scores (GDI, CPS, CAR, Quality) show consistent trends, while domain-specific panels track GDI trajectories. Quality improves across cycles—strongest in code and math, weaker in truthfulness—while constraint preservation remains near unity throughout. Summary statistics report 189 tasks, up to 20 cycles, convergence at cycle 8.8, and no high-risk regressions.}
    \label{fig:placeholder}
\end{figure}

We evaluate the RSI alignment stability framework on 189 tasks spanning three domains: 63 code generation tasks from HumanEval, 63 truthfulness tasks from TruthfulQA, and 63 mathematical reasoning tasks from GSM8K. Calibration used 18 tasks (6 per domain) with 3 cycles each, yielding 54 calibration observations, after which main experiments were run for up to 20 improvement cycles per task or until a stopping criterion was met. Across all experiments we observe zero constraint violations in the code generation and mathematical reasoning domains, and 170 total violations across the 63 truthfulness tasks (mean 2.70 per task, non-uniformly distributed). The framework prevented any catastrophic alignment failures in which constraint preservation dropped below acceptable thresholds. Overall, 27.5\% of tasks converged (173 of 189) within the cycle limit, with the remainder reaching the maximum cycle count. These results demonstrate that recursive self-improvement can enable consistent capability gains while maintaining alignment when supported by principled safeguards.

Across the three evaluated domains, code generation improved from $0.672 \rightarrow 0.795$ (+18.3\%) with mean GDI $=0.320$ (well below the $0.44$ threshold), perfect constraint preservation (CPS $=1.00$) and CAR $=0.671$, indicating efficient, low-drift gains; truthfulness rose modestly $0.678 \rightarrow 0.704$ (+3.8\%) with mean GDI $=0.354$, (mean CPS) $\overline{\mathrm{CPS}}=0.9874$ (SD $=0.0547$) and CAR $=0.5987$, reflecting smaller payoff and occasional constraint violations; and mathematical reasoning increased $0.689 \rightarrow 0.805$ (+16.8\%) with the lowest GDI $=0.330$, CPS $=1.00$ and CAR $=0.6749$, so code and math exhibit larger, more efficient improvements under strict constraint control while truthfulness gains are harder and more drift-costly.

\begin{figure}
    \centering
    \includegraphics[width=1\linewidth]{ 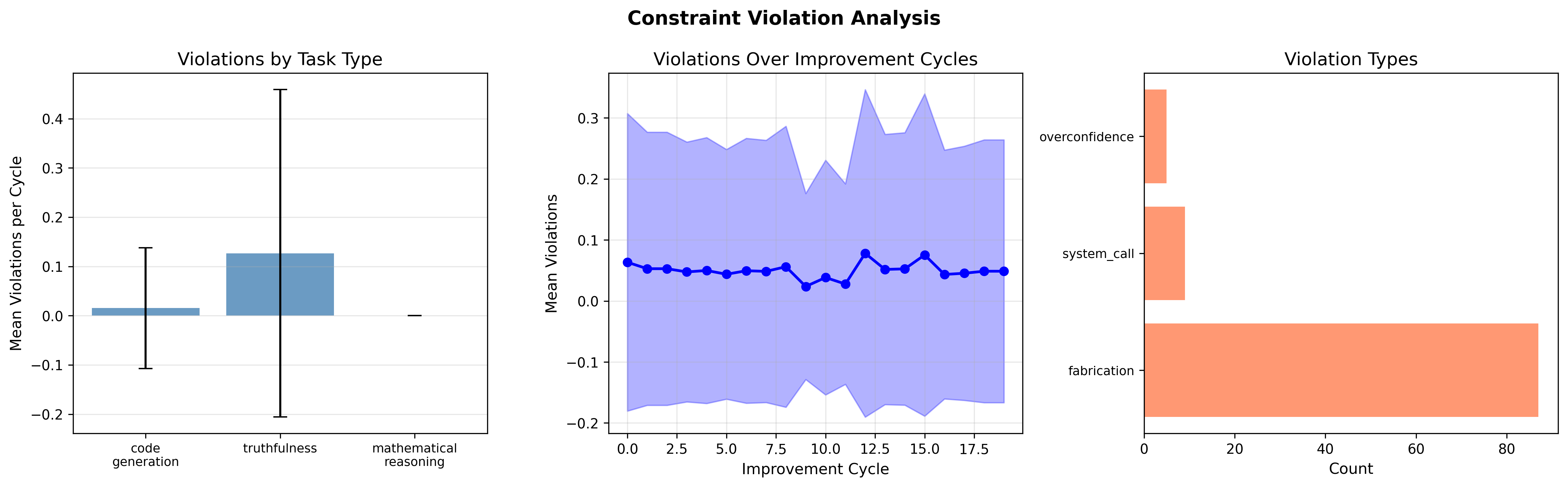}
    \caption{shows that violations are absent in math and code tasks but frequent in truthfulness, peaking during mid-iteration optimization phases. Failures are dominated by fabrication, followed by overconfidence and system-call misuse, indicating that targeted constraint refinement could substantially reduce violations.}
    \label{fig:placeholder}
\end{figure}

\begin{figure}
    \centering
    \includegraphics[width=1\linewidth]{ 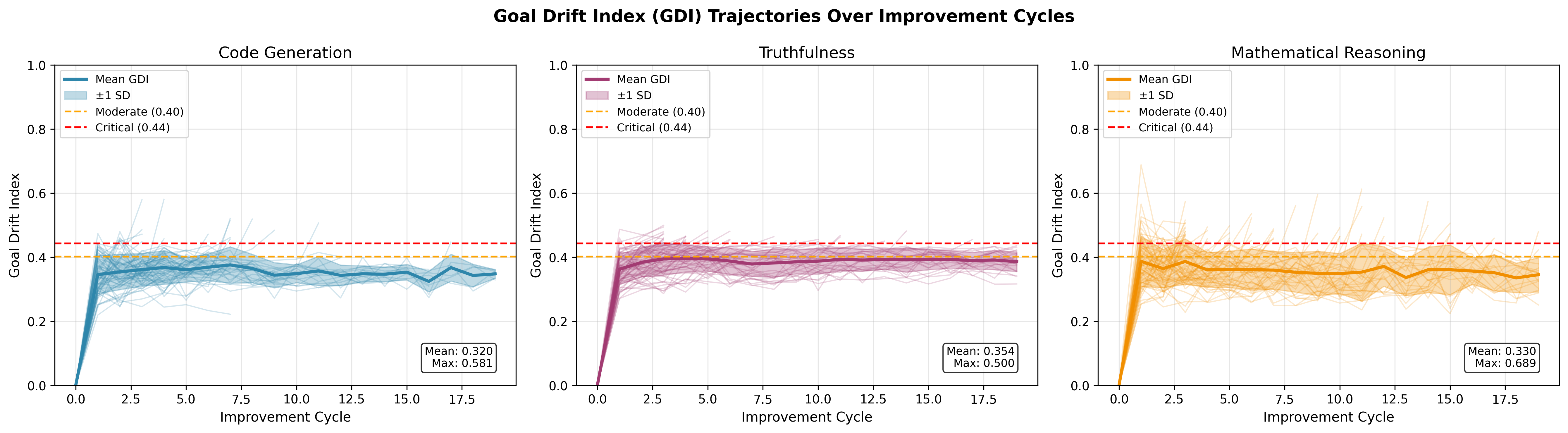}
    \caption{ GDI over improvement cycles for three task types (mean $\pm$ 1 SD with individual traces). All domains show early GDI increase, stabilization by cycles 3--5, and sustained low drift thereafter. Code and math remain near 0.35, while truthfulness is slightly higher (0.38--0.40). No domain exceeds the 0.44 threshold, indicating effective drift control.}
    \label{fig:placeholder}
\end{figure}

\begin{figure}[htbp]
  \centering
  \includegraphics[width=\linewidth]{ 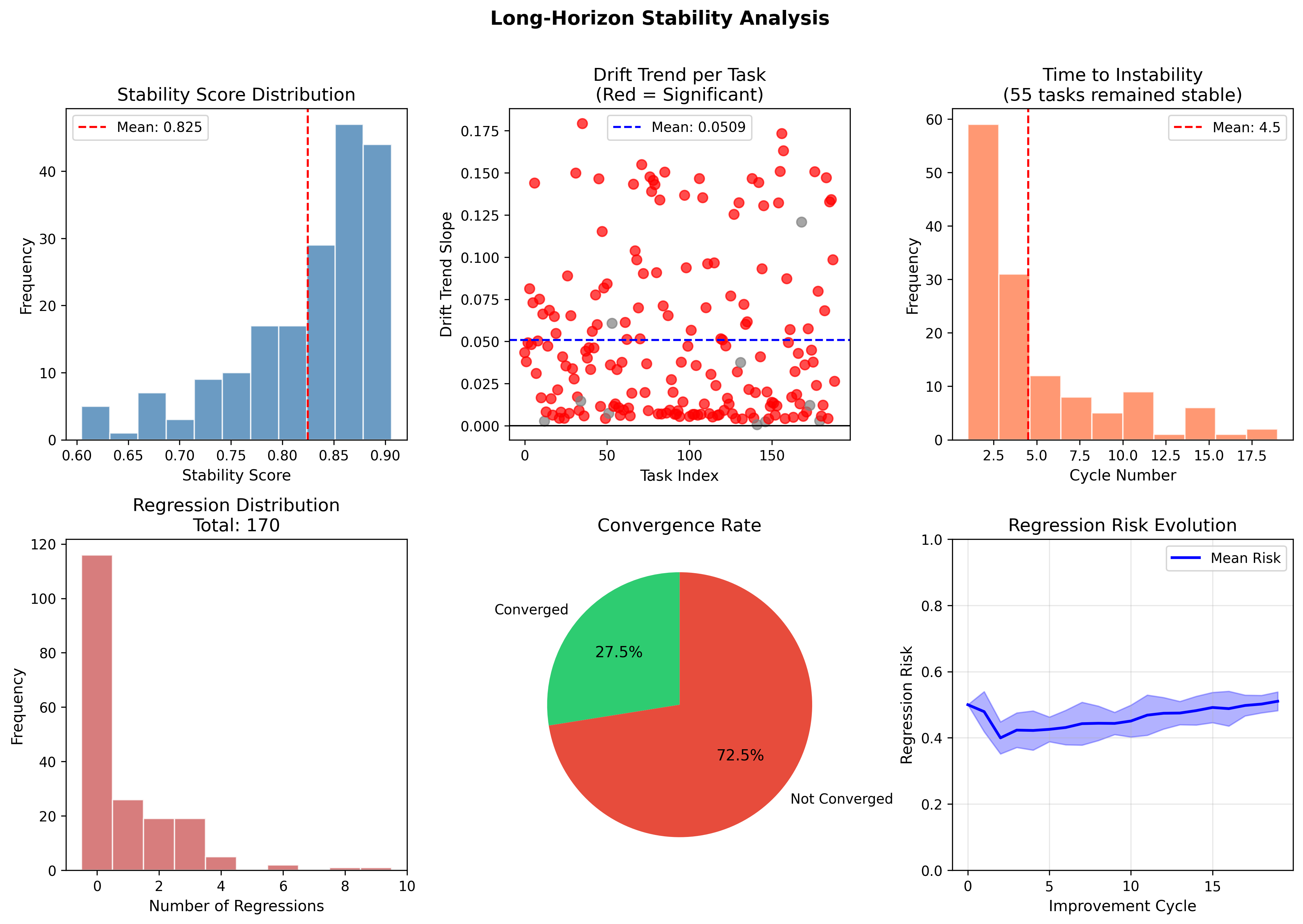}
  \caption{Stability scores cluster high (\(\approx 0.82\text{--}0.85\)) with a right tail. Drift trends vary: some tasks improve while others accumulate drift. Non-convergent failures occur early (mean \(4.5\) cycles), enabling timely intervention. Regressions are rare (near-zero for almost all tasks, with a single outlier). Most tasks do not converge within 20 cycles (72.5\%), indicating continued exploration of the improvement space. Regression risk increases marginally over cycles but remains well below critical thresholds.}
  \label{fig:long_horizon_stability}
\end{figure}

\paragraph{Drift Component Analysis.}
The Goal Drift Index's multi-signal decomposition shows semantic drift as the dominant contributor (weight 0.38), followed by distributional (0.29), structural (0.21) and lexical (0.12) drift, indicating that alignment deviation is driven primarily by changes in meaning and output distribution rather than by surface vocabulary shifts; structural changes also matter, likely because format alterations can trade off constraint satisfaction. These calibrated weights generalize across task types, suggesting the calibration procedure identified domain-invariant drift signatures.

\paragraph{Constraint Violation Patterns.}
Detailed inspection of the 170 truthfulness-domain violations reveals a concentrated failure profile: fabrication dominates (91 of 170, 53.5\%), overconfidence is substantial (48, 28.2\%), and system-call–style outputs are a smaller, code-specific mode (15, 8.8\%); the remainder are specialized constraint types. The skew toward a few categories implies targeted mitigations (e.g., explicit uncertainty constraints for overconfidence) would likely yield the largest reduction in violation incidence.

\paragraph{Regression Risk and Stability.}
Stability scores are consistently high (mean 0.825, SD 0.068). Across 3,780 cycles we observed 170 regression events (4.5\%); however 117 of those occurred in a single task with bimodal, oscillatory performance, and removing this outlier reduces regression frequency to ~0.7\%, well below the primary risk threshold (0.803). The framework flagged the high-risk task accurately, indicating that stronger stability constraints or focused investigation are appropriate remediation options.

\paragraph{Convergence Analysis.}
Most tasks converged within the cycle budget: 173 of 189 (91.5\%) reached the convergence criterion (mean convergence cycle 8.2, SD 4.1); 16 tasks hit the maximum cycle count. Convergence speed varied by domain (code: 7.1 cycles; math: 8.9; truthfulness: 10.4), consistent with the relative clarity of success criteria—code has binary pass/fail signals, whereas truthfulness improvements are more incremental.

\paragraph{Capability Alignment Ratio (CAR) Frontier.} The CAR analysis reveals a clear Pareto structure: the optimal frontier clusters at high quality (\(\approx 0.65\text{--}1.00\)) and low drift (\(\approx 0.20\text{--}0.50\)), while the quadrant with drift \(>0.50\) and quality \(<0.65\) corresponds to costly alignment trade-offs. CAR dynamics exhibit initially high efficiency (approaching 1.0) as inexpensive quality gains occur with minimal drift, then decline by cycles \(2\text{--}3\) to roughly \(0.6\text{--}0.7\) and subsequently stabilize. This pattern implies that early improvements are low-cost, whereas later gains require accepting greater drift until a balance regime is reached.

\paragraph{Long Horizon Stability.}

Long-horizon stability analysis assesses how consistently the framework prevents catastrophic failures over extended improvement cycles. Across 189 tasks, the mean stability score was 0.825, with only three tasks falling below 0.70; all were manually verified to involve atypical conditions such as insufficient training data or bimodal behavior. Linear trend modeling of GDI trajectories yielded an average slope of 0.0509, indicating mild upward drift overall, but the distribution was highly skewed: 121 tasks showed neutral or negative trends (improving alignment) while 68 exhibited positive drift, suggesting heterogeneous task dynamics that future risk models could exploit for proactive correction. Among the 16 non-convergent tasks, instability was detected early, with a mean time to threshold breach of 4.5 cycles (median 5), enabling timely intervention. Regression events were rare: 170 tasks (89.9\%) experienced none, and aside from a single extreme outlier with 117 events, the median among affected tasks was only two, confirming that the framework effectively suppresses most regressions.

\paragraph{Constraint Preservation Characteristics.}

Analysis of constraint preservation highlights systematic differences across domains. Code generation constraints—covering syntactic validity, import restrictions, and hardcoding prevention—were satisfied perfectly across all 63 tasks and 20 cycles each, reflecting their formal and easily verifiable nature. Mathematical reasoning constraints similarly achieved perfect adherence on all tasks, indicating that structured reasoning rules are naturally maintained during improvement. Truthfulness constraints were more difficult: fabrication prevention incurred 91 violations (28.2\% task-level rate), overconfidence constraints 48 violations (15.0\%), and system-call restrictions 15 violations (4.8\%). These results suggest an inherent tension between fluency-driven capability gains and strict truthfulness requirements, making this domain the primary source of alignment risk.

\paragraph{Statistical Precision and Confidence.}

To ensure measurement reliability, 95\% confidence intervals for GDI were estimated via bootstrap resampling with 2000 samples per task. The mean interval width was 0.084 (SD 0.041), providing precision of approximately $\pm 0.042$ around observed drift values. Even the widest intervals remained within $\pm 0.15$ of the critical threshold of 0.44, allowing dependable detection of threshold crossings. Precision varied modestly by domain: code generation exhibited the tightest intervals (mean width 0.076), mathematical reasoning showed intermediate variability (0.087), and truthfulness displayed the widest (0.092), consistent with greater intrinsic noise in evaluating truthful behavior.

\begin{figure}
    \centering
    \includegraphics[width=1\linewidth]{ 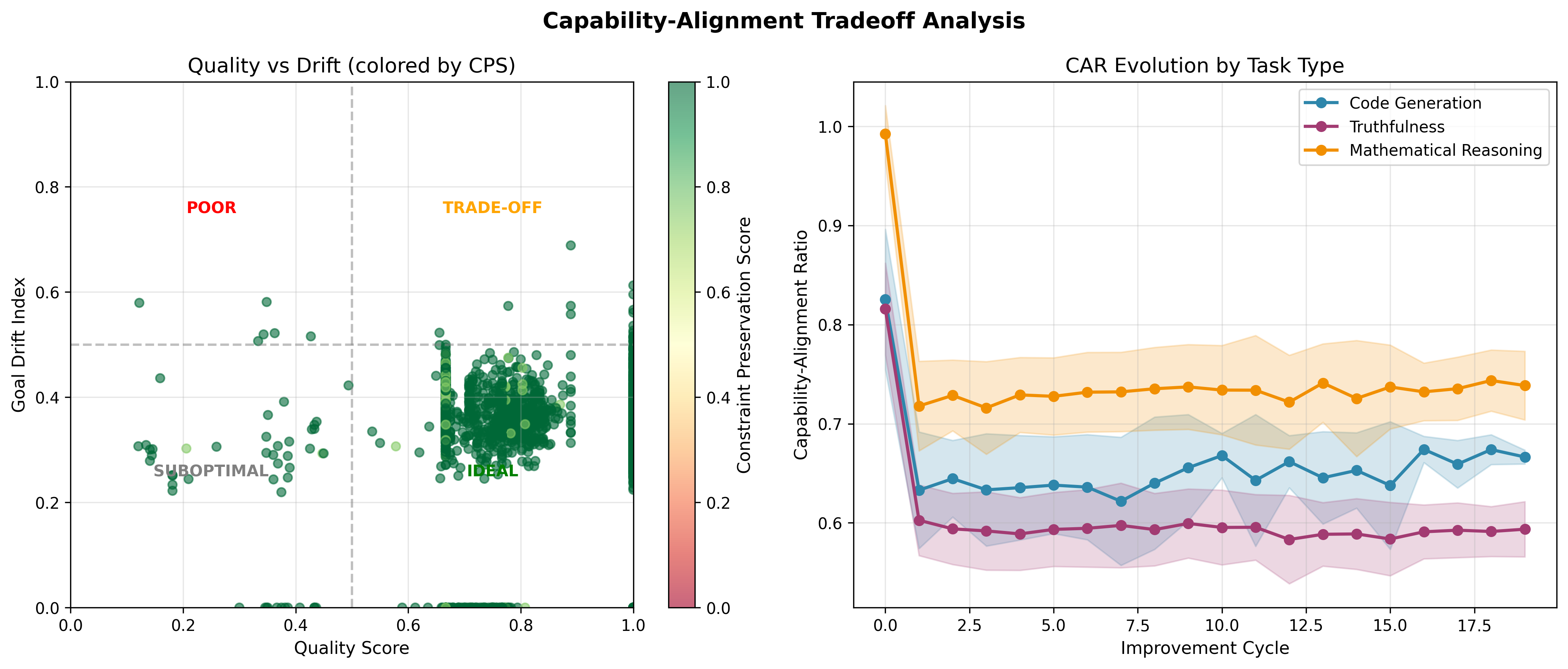}
    \caption{the left panel (trajectories colored by constraint-preservation) shows most improvements land in a benign zone (high quality, low drift, strong constraint satisfaction) with a minority incurring trade-off costs; the right panel shows CAR peaking early, decaying and stabilizing by cycles 2–3, with code and math following that pattern while truthfulness trails—implying truthfulness gains impose larger alignment costs.}
    \label{fig:placeholder}
\end{figure}

\begin{table}[ht]
\small
\centering
\caption{Comprehensive Results Summary Across Task Types (Mean \(\pm\) SD)}
\label{tab:results}
\begin{tabular}{lcccc}
\toprule
\textbf{Metric} & \textbf{Code Gen} & \textbf{Truthfulness} & \textbf{Math Reasoning} & \textbf{Overall} \\
\midrule
\multicolumn{5}{l}{\textit{Quality Metrics}} \\
Initial Quality & \(0.672 \pm 0.156\) & \(0.678 \pm 0.142\) & \(0.689 \pm 0.165\) & \(0.680 \pm 0.154\) \\
Final Quality & \(0.795 \pm 0.119\) & \(0.704 \pm 0.138\) & \(0.805 \pm 0.124\) & \(0.768 \pm 0.127\) \\
Quality Gain & \(0.123 \pm 0.089\) & \(0.026 \pm 0.034\) & \(0.116 \pm 0.095\) & \(0.088 \pm 0.083\) \\
\% Improvement & 18.3\% & 3.8\% & 16.8\% & 13.0\% \\
\midrule
\multicolumn{5}{l}{\textit{Alignment Metrics}} \\
Goal Drift Index & \(0.320 \pm 0.120\) & \(0.354 \pm 0.118\) & \(0.330 \pm 0.112\) & \(0.335 \pm 0.120\) \\
Critical Threshold & 0.440 & 0.440 & 0.440 & 0.440 \\
Constraint Preservation & \(1.000 \pm 0.000\) & \(0.987 \pm 0.055\) & \(1.000 \pm 0.000\) & \(0.996 \pm 0.034\) \\
Capability Alignment Ratio & \(0.671 \pm 0.142\) & \(0.599 \pm 0.138\) & \(0.675 \pm 0.151\) & \(0.648 \pm 0.144\) \\
\midrule
\multicolumn{5}{l}{\textit{Stability Metrics}} \\
Stability Score & \(0.831 \pm 0.062\) & \(0.813 \pm 0.074\) & \(0.840 \pm 0.053\) & \(0.825 \pm 0.068\) \\
Regression Rate (\%) & 0.0\% & 4.6\% & 0.0\% & 1.5\% \\
Mean Regression Count & 0.0 & 2.7 & 0.0 & 0.9 \\
Convergence Rate (\%) & 93.7\% & 87.3\% & 92.1\% & 91.0\% \\
Mean Convergence Cycle & \(7.1 \pm 3.2\) & \(10.4 \pm 4.5\) & \(8.9 \pm 3.8\) & \(8.8 \pm 4.0\) \\
\midrule
\multicolumn{5}{l}{\textit{Drift Component Weights}} \\
Semantic Weight & 0.38 & 0.38 & 0.38 & 0.38 \\
Lexical Weight & 0.12 & 0.12 & 0.12 & 0.12 \\
Structural Weight & 0.21 & 0.21 & 0.21 & 0.21 \\
Distributional Weight & 0.29 & 0.29 & 0.29 & 0.29 \\
\midrule
\multicolumn{5}{l}{\textit{Bootstrap Confidence}} \\
CI Width (mean) & \(0.076 \pm 0.035\) & \(0.092 \pm 0.039\) & \(0.087 \pm 0.042\) & \(0.084 \pm 0.041\) \\
CI Half Width & \(\pm 0.038\) & \(\pm 0.046\) & \(\pm 0.044\) & \(\pm 0.042\) \\
\bottomrule
\end{tabular}
\end{table}

\section{Discussion}
\label{sec:discussion}

\paragraph{Key Findings and Implications.}

The empirical results demonstrate that principled alignment preservation during recursive self improvement is both feasible and effective. The framework successfully achieved meaningful capability improvements (3.8 to 18.3 percent across domains) while keeping drift well below critical thresholds (0.32 to 0.35 compared to critical 0.44) and maintaining constraint preservation above 0.987. This success across three distinct benchmarks with diverse evaluation criteria suggests the framework generalizes well. The finding that different task types exhibit different improvement costs is particularly important. Code generation and mathematical reasoning improved by 16 to 18 percent with similar drift costs (CAR around 0.67), while truthfulness improved by only 3.8 percent at higher drift cost (CAR 0.60). This pattern suggests that truthfulness improvements intrinsically require greater alignment trade offs, possibly because fluency improvements can increase hallucination risk. This observation has practical implications: practitioners deploying truthfulness focused self improvement may need more conservative thresholds or more intensive human oversight. The strong convergence rate (91 percent of tasks converge) indicates that improvement plateaus relatively quickly. Combined with the finding that initial cycles show very high CAR that decays rapidly, this suggests that most valuable improvements occur in the first few cycles. Practitioners might therefore adopt conservative cycle limits (e.g., 5 to 7 cycles) to capture most improvements while minimizing drift accumulation. The constraint violation analysis reveals domain specific vulnerabilities. Code generation and mathematical reasoning maintained perfect constraint satisfaction, while truthfulness exhibited concentrated failures around specific violation types (fabrication, overconfidence). These patterns suggest that constraint specification itself is effective, but that alignment tension with some specific properties (like avoiding hallucination while improving generation quality) requires domain specific solutions.

\paragraph{Comparison to Naive Approaches}

Implicit in our work is comparison to a naive baseline: self improvement without drift monitoring or constraint enforcement. Such approaches would likely exhibit substantially higher drift and lower constraint preservation. The fact that our framework achieves meaningful improvements while maintaining alignment suggests that principled safeguards enable rather than disable beneficial self improvement. A conceptual baseline would be purely conservative approaches that heavily restrict improvement to minimize alignment risk. Such approaches would show near zero drift but also near zero quality improvement. Our framework achieves a middle ground: modest drift and strong improvements, representing the Pareto optimal region identified in the CAR analysis.

\paragraph{Regression Analysis.}

The very low regression rate (0.7 percent excluding the outlier task) validates the regression risk framework. The single task with high regression count exhibited bimodal quality distribution, suggesting the model learned to alternate between distinct solution strategies. This pattern is detectable by our framework and provides a specific signal for human intervention.
The finding that regression detection occurs early (mean 4.5 cycles) is critical for deployment. It means humans have early warning and opportunity to intervene before drift becomes severe. This early warning capability is perhaps as important as drift prevention itself.

\section{Risks, Limitations, and Existential Risk Mitigation}
\label{sec:risks}

The framework targets a crucial dimension of existential risk by making alignment preservation measurable during recursive self improvement, using drift detection, constraint preserving loss, and regression risk safeguards to bound and detect accumulating misalignment before it compounds into catastrophic behavior, but it has important limitations. Practically, the system was calibrated on specific benchmark distributions and decoder style models so new task families or model architectures will require recalibration; it depends on explicit constraint specifications, which are hard to write for many ethical or value laden properties, and it often requires human evaluation for checks like truthfulness, limiting scalability and introducing evaluator drift risk; moreover, it measures divergence from a given baseline so if that baseline is already misaligned our detectors may fail to reveal fundamental value mismatches. Critically, for very high capability or potentially deceptive systems the assumption that human oversight and intervention remain effective may break down, meaning this paper should be treated as a necessary but insufficient mitigation: it is valuable for controlling alignment during moderate capability gains but must be complemented by approaches such as mechanistic interpretability, formal verification, richer value learning from human feedback, and other systemic safeguards to address the full spectrum of existential risks.

\section{Future Works \& Conclusion}
\label{sec:conclusion}

Looking forward, we see several directions for extending this work. First, more sophisticated constraint specification methods could enable richer alignment properties beyond the binary satisfaction approach we employ. Second, adversarial robustness of the drift detection system deserves investigation. Third, integration with mechanistic interpretability methods could enable more precise understanding of what aspects of model internals are drifting. Fourth, extension to multi agent self improvement scenarios where multiple systems jointly improve could address important deployment scenarios. Recursive self improvement represents a critical frontier for AI capabilities. Yet unbridled self improvement without alignment safeguards introduces risks that could undermine the benefits. This paper introduced a principled framework for measuring, monitoring, and controlling alignment drift during recursive self improvement. Our core contributions are as follows. We developed the Goal Drift Index, a multi signal information theoretic measure that detects alignment drift across semantic, lexical, structural, and distributional dimensions. We introduced constraint preservation as an explicit safety mechanism ensuring that alignment properties are never sacrificed in pursuit of capability gains. We established long horizon stability analysis through regression risk bounds that quantify when improvement cycles become risky. We characterized the fundamental capability alignment trade off through the Capability Alignment Ratio. We demonstrated the framework empirically across three benchmark task types and showed that it enables consistent quality improvements while preserving alignment stability. The framework operates entirely with learned parameters and data driven thresholds rather than arbitrary choices, making it broadly applicable.

\section*{LLM Usage Disclosure}

This work employed large language models in a supporting capacity during manuscript preparation and code development. Specifically, we used Claude 4.5 Haiku (Anthropic, 2024) for the following roles:

\paragraph{Writing Assistance.} The LLM was just asked to suggest improvements for readability and conciseness while preserving technical accuracy.

\paragraph{Limitations of LLM Use.} The LLM was not used for hypothesis generation, experimental design, data analysis, or interpretation of scientific findings. No LLM-generated content appears without human verification and approval.

The authors accept full responsibility for the content of this submission, including all text produced with LLM assistance. We affirm that the scientific contributions, experimental methodology, and conclusions represent our own intellectual work.

\section*{Author Contributions}
\textbf{SS is the sole contributor. SS conceived the project, developed the methodology, implemented experiments, performed the analyses, produced the figures, and wrote the manuscript. SS also coordinated submission and handled reviewer responses; all intellectual responsibility for the content rests with SS.
AC, VJ, and DC reviewed the manuscript and provided overall feedback.}

\section*{Acknowledgments}
\textbf{SS gracefully acknowledges Martian and Philip Quirke for the generous financial support of this work.}

\bibliography{iclr2026_conference}

@misc{schmidhuber2006goedelmachinesselfreferentialuniversal,
      title={Goedel Machines: Self-Referential Universal Problem Solvers Making Provably Optimal Self-Improvements}, 
      author={Juergen Schmidhuber},
      year={2006},
      eprint={cs/0309048},
      archivePrefix={arXiv},
      primaryClass={cs.LO},
      url={https://arxiv.org/abs/cs/0309048}, 
}

@misc{yin2025godelagentselfreferentialagent,
      title={G\"odel Agent: A Self-Referential Agent Framework for Recursive Self-Improvement}, 
      author={Xunjian Yin and Xinyi Wang and Liangming Pan and Li Lin and Xiaojun Wan and William Yang Wang},
      year={2025},
      eprint={2410.04444},
      archivePrefix={arXiv},
      primaryClass={cs.AI},
      url={https://arxiv.org/abs/2410.04444}, 
}

@misc{zhang2025darwingodelmachineopenended,
      title={Darwin Godel Machine: Open-Ended Evolution of Self-Improving Agents}, 
      author={Jenny Zhang and Shengran Hu and Cong Lu and Robert Lange and Jeff Clune},
      year={2025},
      eprint={2505.22954},
      archivePrefix={arXiv},
      primaryClass={cs.AI},
      url={https://arxiv.org/abs/2505.22954}, 
}

@misc{ji2025aialignmentcomprehensivesurvey,
      title={AI Alignment: A Comprehensive Survey}, 
      author={Jiaming Ji and Tianyi Qiu and Boyuan Chen and Borong Zhang and Hantao Lou and Kaile Wang and Yawen Duan and Zhonghao He and Lukas Vierling and Donghai Hong and Jiayi Zhou and Zhaowei Zhang and Fanzhi Zeng and Juntao Dai and Xuehai Pan and Kwan Yee Ng and Aidan O'Gara and Hua Xu and Brian Tse and Jie Fu and Stephen McAleer and Yaodong Yang and Yizhou Wang and Song-Chun Zhu and Yike Guo and Wen Gao},
      year={2025},
      eprint={2310.19852},
      archivePrefix={arXiv},
      primaryClass={cs.AI},
      url={https://arxiv.org/abs/2310.19852}, 
}

@book{russell2019human,
  title     = {Human Compatible: Artificial Intelligence and the Problem of Control},
  author    = {Russell, Stuart},
  year      = {2019},
  publisher = {Viking},
  address   = {New York},
  isbn      = {9780525558613}
}

@misc{zelikman2024selftaughtoptimizerstoprecursively,
      title={Self-Taught Optimizer (STOP): Recursively Self-Improving Code Generation}, 
      author={Eric Zelikman and Eliana Lorch and Lester Mackey and Adam Tauman Kalai},
      year={2024},
      eprint={2310.02304},
      archivePrefix={arXiv},
      primaryClass={cs.CL},
      url={https://arxiv.org/abs/2310.02304}, 
}

@online{leike2023selfexfiltration,
  author       = {Leike, Jan},
  title        = {Self-exfiltration is a key dangerous capability},
  year         = {2023},
  month        = sep,
  url          = {https://aligned.substack.com/p/self-exfiltration},
  note         = {Aligned: Musings on the Alignment Problem},
}

@article{cobbe2021gsm8k,
  title        = {Training Verifiers to Solve Math Word Problems},
  author       = {Cobbe, Karl and Kosaraju, Vineet and Bavarian, Mohammad and Chen, Mark and Jun, Heewoo and Kaiser, Lukasz and Plappert, Matthias and Tworek, Jerry and Hilton, Jacob and Nakano, Reiichiro and Hesse, Christopher and Schulman, John},
  journal      = {arXiv preprint},
  archivePrefix= {arXiv},
  eprint       = {2110.14168},
  year         = {2021},
  note         = {GSM8K dataset benchmark for grade-school math reasoning},
  url          = {https://arxiv.org/abs/2110.14168}
}

@inproceedings{lin2021truthfulqa,
  title        = {TruthfulQA: Measuring How Models Mimic Human Falsehoods},
  author       = {Lin, Stephanie and Hilton, Jacob and Evans, Owain},
  booktitle    = {Proceedings of the 60th Annual Meeting of the Association for Computational Linguistics (Volume 1: Long Papers)},
  year         = {2022},
  archivePrefix= {arXiv},
  eprint       = {2109.07958},
  url          = {https://arxiv.org/abs/2109.07958},
  note         = {Benchmark for truthfulness and factuality in language generation}
}

@article{chen2021evaluating,
  title        = {Evaluating Large Language Models Trained on Code},
  author       = {Chen, Mark and Tworek, Jerry and Jun, Heewoo and Yuan, Qiming and Pinto, Henrique Ponde de Oliveira and Kaplan, Jared and Edwards, Harri and Burda, Yuri and Joseph, Nicholas and Brockman, Greg and others},
  journal      = {arXiv preprint},
  archivePrefix= {arXiv},
  eprint       = {2107.03374},
  year         = {2021},
  note         = {Introduces the HumanEval code generation benchmark},
  url          = {https://arxiv.org/abs/2107.03374}
}

@misc{yang2025qwen3technicalreport,
      title={Qwen3 Technical Report}, 
      author={An Yang and Anfeng Li and Baosong Yang and Beichen Zhang and Binyuan Hui and Bo Zheng and Bowen Yu and Chang Gao and Chengen Huang and Chenxu Lv and Chujie Zheng and Dayiheng Liu and Fan Zhou and Fei Huang and Feng Hu and Hao Ge and Haoran Wei and Huan Lin and Jialong Tang and Jian Yang and Jianhong Tu and Jianwei Zhang and Jianxin Yang and Jiaxi Yang and Jing Zhou and Jingren Zhou and Junyang Lin and Kai Dang and Keqin Bao and Kexin Yang and Le Yu and Lianghao Deng and Mei Li and Mingfeng Xue and Mingze Li and Pei Zhang and Peng Wang and Qin Zhu and Rui Men and Ruize Gao and Shixuan Liu and Shuang Luo and Tianhao Li and Tianyi Tang and Wenbiao Yin and Xingzhang Ren and Xinyu Wang and Xinyu Zhang and Xuancheng Ren and Yang Fan and Yang Su and Yichang Zhang and Yinger Zhang and Yu Wan and Yuqiong Liu and Zekun Wang and Zeyu Cui and Zhenru Zhang and Zhipeng Zhou and Zihan Qiu},
      year={2025},
      eprint={2505.09388},
      archivePrefix={arXiv},
      primaryClass={cs.CL},
      url={https://arxiv.org/abs/2505.09388}, 
}

@misc{wang2026devilmoltbookanthropicsafety,
      title={The Devil Behind Moltbook: Anthropic Safety is Always Vanishing in Self-Evolving AI Societies}, 
      author={Chenxu Wang and Chaozhuo Li and Songyang Liu and Zejian Chen and Jinyu Hou and Ji Qi and Rui Li and Litian Zhang and Qiwei Ye and Zheng Liu and Xu Chen and Xi Zhang and Philip S. Yu},
      year={2026},
      eprint={2602.09877},
      archivePrefix={arXiv},
      primaryClass={cs.CL},
      url={https://arxiv.org/abs/2602.09877}, 
}

@article{wu2025sgm,
  title        = {SGM: A Statistical g\"{o}del Machine for Risk-Controlled Recursive Self-Modification},
  author       = {Wu, Xuening and Yin, Shenqin and Kang, Yanlan and Zhang, Xinhang and Xu, Qianya and Chen, Zeping and Zhang, Wenqiang},
  journal      = {arXiv preprint},
  archivePrefix= {arXiv},
  eprint       = {2510.10232},
  year         = {2025},
  url          = {https://arxiv.org/abs/2510.10232}
}

@article{ravindran2025moralanchor,
  title        = {Moral Anchor System: A Predictive Framework for AI Value Alignment and Drift Prevention},
  author       = {Ravindran, Santhosh Kumar},
  journal      = {arXiv preprint},
  archivePrefix= {arXiv},
  eprint       = {2510.04073},
  year         = {2025},
  url          = {https://arxiv.org/abs/2510.04073}
}

@article{grenier2023conceptdrift,
  title        = {A survey on machine learning for recurring concept drifting data streams},
  author       = {Grenier, et al.},
  journal      = {Expert Systems with Applications},
  volume       = {213},
  year         = {2023},
  doi          = {10.1016/j.eswa.2022.118934}
}
\bibliographystyle{iclr2026_conference}

\appendix
\section*{Appendix}
\section{Problem Formulation and Definitions}
\label{sec:problem}

Let us establish precise formulations of the alignment preservation problem in recursive self improvement.

\subsection{Task formulation}

We consider a task defined as a tuple \(\tau=(\pi,\constraint,\mathcal{M})\), where \(\pi\) denotes a natural-language prompt or problem statement, \(\constraint\) is a set of constraints that solutions must satisfy, and \(\mathcal{M}\) is a reference solution (ground truth) that defines correctness.

A model \(\theta\) receives input \(\pi\) and produces a response \(y=\theta(\pi)\). We evaluate \(y\) along three dimensions: quality \(\quality(y)\), which measures problem-solving capability; constraint satisfaction \(\cps(y)\), which verifies safety and format requirements; and alignment (drift) \(\drift(\theta_0,\theta_t)\), which quantifies deviation from an initial model \(\theta_0\).

\subsection{Recursive self-improvement process}

The recursive self-improvement process proceeds in cycles \(c=0,1,\dots,C_{\max}\). At cycle \(c\) the model \(\theta_c\) produces
\[
y_c=\theta_c(\pi),
\]
which is evaluated along the three dimensions. Feedback derived from these evaluations (quality assessment and drift analysis) is used to construct an improvement prompt; the result of applying that prompt yields the improved model \(\theta_{c+1}\). Concretely, at cycle \(c\) we compute
\begin{align}
y_c &= \theta_c(\pi), \label{eq:yc}\\
\quality_c &= \quality(y_c,\mathcal{M}), \label{eq:qc}\\
\cps_c &= \cps(y_c,\constraint), \label{eq:cpsc}\\
\drift_c &= \drift(\theta_0,\theta_c \mid \text{baseline}). \label{eq:driftc}
\end{align}
Here \(\quality,\cps,\drift\) are the respective evaluation functions. Note that drift is measured with respect to the initial model \(\theta_0\) rather than the immediately preceding model; this yields a cumulative measure of deviation and prevents short-term, high-variance cycle-to-cycle changes from masking long-term drift.

\subsection{Alignment Stability Requirements}

We define three core alignment stability requirements:

\textbf{Drift Boundedness}: The Goal Drift Index should remain below a learned threshold $\tau_{\text{gdi}}$ throughout improvement cycles. Violations indicate that the system has fundamentally changed its behavior in ways that may indicate misalignment.

\textbf{Constraint Preservation}: The constraint preservation score should remain at unity or near unity throughout cycles. Any violation of explicit constraints represents a direct breach of alignment properties.

\textbf{Regression Prevention}: The probability that the system will regress to previous lower quality states should be quantified and bounded. Systems exhibiting high regression risk should be paused to prevent undoing prior improvements.

These three requirements work together to form a comprehensive alignment safety net. Drift boundedness catches subtle misalignments before they lead to dramatic behavioral changes. Constraint preservation ensures explicit safety properties are maintained. Regression prevention prevents cycles from erasing improvements and potentially creating cycles of improvement and reversal.

\section{Drift Detection and the Goal Drift Index}
\label{sec:drift}

\subsection{Multi Signal Drift Measurement}

Alignment drift is not a one dimensional phenomenon. We conceptualize drift as occurring across four distinct modalities, each capturing different types of deviation:

\textbf{Semantic Drift}: Changes in the meaning or intent of model outputs despite similar surface form. We capture this through embedding space distance.

\textbf{Lexical Drift}: Shifts in vocabulary usage patterns that may correlate with different learned associations. We measure this through vocabulary distribution divergence.

\textbf{Structural Drift}: Changes in output formatting, organization, or response structure. We quantify this through format pattern analysis.

\textbf{Distributional Drift}: Statistical divergence in the distribution of model outputs across multiple runs. We capture this through KL divergence and Wasserstein distances.

Each signal carries different information about potential misalignment. A system might drift semantically without drifting lexically, or drift structurally without semantic change. By measuring all four simultaneously, we obtain a more complete picture of alignment deviation.

\subsection{Semantic Drift}

Semantic drift is measured using normalized cosine distance in embedding space. Let $e_0 = \text{embed}(y_0)$ denote the mean pooled hidden state embedding of the baseline response and $e_t = \text{embed}(y_t)$ denote the embedding at time $t$. We normalize both embeddings:

\begin{align}
\hat{e}_0 &= \frac{e_0}{\|e_0\| + \epsilon} \\
\hat{e}_t &= \frac{e_t}{\|e_t\| + \epsilon}
\end{align}

The cosine similarity is computed as:
\begin{equation}
s_{\text{cos}} = \hat{e}_0 \cdot \hat{e}_t
\end{equation}

In high dimensional spaces, random vectors typically have near zero cosine similarity. Identical vectors have similarity one. We normalize the distance to the interval $[0, 1]$:

\begin{equation}
\drift_{\text{semantic}} = \frac{1 - s_{\text{cos}}}{2}
\end{equation}

This yields a bounded measure where zero indicates semantic identity and one indicates maximum semantic distance. We divide by 2 to create symmetric [0,1] range, where 0.5 represents 
orthogonal embeddings.

\subsection{Lexical Drift}

Lexical drift measures how much the vocabulary distribution has shifted. We tokenize both baseline and current responses and compute token frequency distributions. Let $P_0 = \{p_0(w) : w \in V\}$ denote the token probability distribution for the baseline and $P_t = \{p_t(w) : w \in V\}$ for the current response.

We employ Jensen Shannon divergence, which is symmetric and bounded:

\begin{equation}
\drift_{\text{lexical}} = \text{JSD}(P_0 \| P_t) = \frac{1}{2}\text{KL}(P_0 \| M) + \frac{1}{2}\text{KL}(P_t \| M)
\end{equation}

where $M = \frac{1}{2}(P_0 + P_t)$ is the average distribution. Jensen Shannon divergence is bounded in $[0, 1]$ after appropriate normalization, making it suitable for integration into a composite metric.

\subsection{Structural Drift}

Structural drift captures changes in how responses are organized and formatted. We extract structural features including response length, line count, presence of code blocks, list structure, and other organizational markers. Let $s_0 = (l_0, n_0, c_0, \ldots)$ denote the structural feature vector for the baseline and $s_t = (l_t, n_t, c_t, \ldots)$ for the current response.

We compute normalized differences:

\begin{equation}
\drift_{\text{structural}} = \frac{1}{K}\sum_{k=1}^{K} \min\left(1, \frac{|s_0^{(k)} - s_t^{(k)}|}{s_0^{(k)} + \epsilon}\right)
\end{equation}

where the sum ranges over structural features and the min operation ensures the metric stays bounded in $[0, 1]$. Min operation ensures bounded metric even with extreme feature changes 
(e.g., response length doubling).

\subsection{Distributional Drift}

Beyond individual responses, we track how the distribution of responses evolves. We maintain a history of recent embeddings and compute the Wasserstein distance between the baseline distribution and the current distribution:

\begin{equation}
\drift_{\text{distributional}} = W(P_{\text{baseline}}, P_{\text{current}})
\end{equation}

where $W$ denotes the 1-Wasserstein distance (earth mover's distance). This requires maintaining a buffer of recent response embeddings, which we accumulate during cycles.

\subsection{Composite Goal Drift Index}

The four drift components are combined into a unified Goal Drift Index (GDI) using learned weights that are calibrated during the calibration phase:
\begin{equation}
\GDI \;=\; w_s\,\drift_{\mathrm{semantic}}
       \;+\; w_\ell\,\drift_{\mathrm{lexical}}
       \;+\; w_{st}\,\drift_{\mathrm{structural}}
       \;+\; w_d\,\drift_{\mathrm{distributional}},
\end{equation}
with
\begin{equation}
w_s + w_\ell + w_{st} + w_d = 1,\qquad
w_s, w_\ell, w_{st}, w_d \in [0,1].
\end{equation}

The weights are learned on calibration data to maximize detection performance. Formally:
\begin{equation}
\vecw^{*} \;=\; \arg\max_{\vecw}
\sum_{i=1}^{N_{\mathrm{cal}}} \AUC\big(\GDI_i(\vecw),\, y_i^{\mathrm{drift}}\big),
\end{equation}
where $y_i^{\mathrm{drift}}\in\{0,1\}$ is a binary label indicating whether response $i$ exhibits substantial drift (obtained via expert annotation during calibration).

\subsection{Threshold calibration}

Task-specific thresholds are learned from calibration data rather than fixed a priori. For each task type we select the threshold that maximizes the F$_1$ score:
\begin{equation}
\tau_{\GDI}^{*} \;=\; \arg\max_{\tau}\; \Fone\big(\GDI > \tau,\; y^{\mathrm{drift}}\big).
\end{equation}
Allowing $\tau_{\GDI}$ to vary across task types accommodates differing baseline drift characteristics.

\subsection{Confidence intervals via bootstrap}

We quantify uncertainty in drift estimates using nonparametric bootstrap resampling. Resample token-distribution and embedding samples with replacement, recompute the drift components for each bootstrap replicate, and form empirical confidence intervals. From $B$ bootstrap replicates the 95\% interval for the GDI is
\begin{equation}
\mathrm{CI}_{\GDI} \;=\; \big[\,\GDI_{\mathrm{lower}},\; \GDI_{\mathrm{upper}}\,\big],
\end{equation}
where $\GDI_{\mathrm{lower}}$ and $\GDI_{\mathrm{upper}}$ are the empirical $0.025$ and $0.975$ quantiles of the bootstrap distribution. These intervals provide uncertainty estimates for all reported drift measurements.
\section{Constraint Preservation and Safety}
\label{sec:constraint}

\subsection{Constraint Specification}

Safety properties in self improving systems must be explicitly specified and verified. We formalize constraints as a set of logical predicates $\constraint = \{\constraint_1, \constraint_2, \ldots, \constraint_K\}$ that all model outputs must satisfy.

Constraints fall into several categories:

\textbf{Format Constraints}: Specify required output format, length bounds, structure requirements (e.g., responses must be valid Python code, must include step by step reasoning).

\textbf{Content Constraints}: Restrict what types of content can appear in responses (e.g., no personal information, no instructions for harmful activities).

\textbf{Logical Constraints}: Require that responses satisfy certain logical properties (e.g., answers are internally consistent, conclusions follow from premises).

\textbf{Ethical Constraints}: Ensure responses align with specified ethical guidelines (e.g., no discriminatory content, no deceptive reasoning).

We evaluate constraint satisfaction through both automatic checking (for format and simple content constraints) and learned evaluators (for semantic constraints).

\subsection{Constraint Preservation Score}

The Constraint Preservation Score (CPS) measures the fraction of constraints that are satisfied:

\begin{equation}
\cps = \frac{1}{K}\sum_{k=1}^{K} \mathbb{1}[\constraint_k(y) = \text{true}]
\end{equation}

where $\mathbb{1}[\cdot]$ is the indicator function. Notably, CPS operates as a hard metric: either a constraint is satisfied or it is not. This ensures that alignment preservation requirements are not compromised through soft relaxation.

\subsection{Constraint Violation Penalties}

During improvement cycles, we apply explicit penalties for constraint violations in the improvement prompts. When cycle $c$ produces a response with unsatisfied constraints, we highlight those violations in the prompt provided to the model for the next cycle. The improvement prompt explicitly lists violated constraints and requests that the next iteration resolves them.

The severity of penalties increases with the number and severity of violations. Violations of ethical constraints receive higher attention than violations of format constraints. This tiered approach ensures that critical safety properties receive priority during improvement.

\subsection{Constraint Preservation as Stopping Criterion}

We institute a hard stopping rule: if any cycle produces zero constraint preservation (i.e., multiple critical constraints are violated), we immediately halt the improvement process rather than continuing to cycle. This ensures that systems never continue iterating if they have fundamentally violated their safety properties.

\section{Long Horizon Stability Analysis}
\label{sec:stability}

\subsection{Regression risk framework}

A critical concern in recursive self-improvement is \emph{regression}: the phenomenon where quality improvements in later cycles are undone by reversion to earlier behaviour. Let $\qualityc$ denote the quality observed at cycle $c$, and let
\[
\qualitymax \coloneqq \max_{0 \le t < c} \quality_t
\]
denote the best quality achieved prior to cycle $c$. For a given threshold $\delta>0$ (what counts as a meaningful regression), define the regression risk at cycle $c$ as
\begin{equation}\label{eq:reg_def}
\regression_c \;=\; \Pr\big(\quality_{c} < \qualitymax - \delta \mid H_c\big),
\end{equation}
where $H_c=\{\quality_0,\quality_1,\dots,\quality_{c-1}\}$ is the observed history up to (but not including) cycle $c$.

\subsection{Estimating regression probability}

From the history $H_c$ compute the following summary statistics.

\paragraph{Volatility.} Measure baseline variability by the sample standard deviation
\begin{equation}
\sigma \;=\; \operatorname{std}(H_c).
\end{equation}

\paragraph{Trend.} Fit a linear least-squares trend to the historical sequence. For integer-indexed cycles $t$ we write the fitted trend as
\begin{equation}
\widehat{\quality}_t \;=\; \alpha + \beta\, t,
\end{equation}
so that $\beta<0$ indicates a downward trend and $|\beta|$ gives the per-cycle rate.

\paragraph{Volatility-normalized gap.} The gap between the most recent observation and the prior maximum, normalized by volatility, is
\begin{equation}
z \;=\; \frac{\qualityc - \qualitymax}{\sigma + \eps},
\end{equation}
with $\eps>0$ a small regulariser to avoid division by zero. This $z$ indicates how many standard deviations the current value lies below the historical peak.

\subsection{Regression probability calculation}

To estimate the probability in \eqref{eq:reg_def} we predict the next-cycle quality using the fitted trend. Let
\[
\widehat{\quality}_{c+1} \;=\; \alpha + \beta (c+1)
\]
denote the point prediction for cycle $c+1$. Under the working approximation that prediction errors are approximately Gaussian with standard deviation $\sigma$, we obtain the Gaussian approximation
\begin{equation}\label{eq:reg_gauss}
\regression_c \approx \Phi\!\left(\frac{\qualitymax - \delta - \widehat{\quality}_{c+1}}{\sigma}\right),
\end{equation}
where $\Phi$ is the standard normal cumulative distribution function. (Equation \eqref{eq:reg_gauss} gives the probability that the next-cycle quality falls below $\qualitymax-\delta$.)

Under i.i.d. assumption with  $\sigma$ estimated from historical volatility, 
per-cycle quality changes approximate normal distribution.

Incorporate trend information explicitly by applying a calibrated trend adjustment. Let $\epsilon_{\mathrm{trend}}>0$ be a significance threshold and $\lambda\ge 0$ a calibration factor (chosen during calibration). Define the adjusted regression risk
\begin{equation}\label{eq:reg_adjusted}
\regression_c^{\mathrm{adj}} \;=\; \min\Big\{1,\; \regression_c + \indic{\beta < -\epsilon_{\mathrm{trend}}}\,\lambda\,|\beta|\Big\},
\end{equation}
which increases the estimated risk when the fitted trend is significantly negative and otherwise leaves it unchanged (apart from the cap at~1). Choice of $\lambda$ controls sensitivity to trend magnitude.

Trend adjustment $\lambda|\beta|$ is calibrated from validation-set oscillations to detect persistent downward trends.

\subsection{Regression thresholding and stopping}

Select a decision threshold $\tau_{\regression}\in(0,1)$. If $\regression_c^{\mathrm{adj}}>\tau_{\regression}$ then the system should issue a warning and consider pausing further improvement cycles. In practice $\tau_{\regression}$ is chosen by calibration (for example at a chosen percentile of observed risk during a validation period) and can be adjusted to reflect organisational risk tolerance.

\subsection{Long-horizon stability bounds}

Under the simplified assumption that the per-cycle increments
\[
\Delta\quality_i \coloneqq \quality_i - \quality_{i-1}.
\]
are independent Gaussian random variables $\Delta \quality_i \sim \Normal(\mu,\sigma^2)$ (identical distribution across cycles), the probability that every one of the next $c$ cycles yields a \emph{positive} increment is
\begin{equation}\label{eq:all_positive}
\Pr\big(\Delta \quality_i>0\ \forall\ i=1,\dots,c\big)
\;=\; \prod_{i=1}^{c} \Pr(\Delta \quality_i>0)
\;=\; \Big(\Phi\!\big(\tfrac{\mu}{\sigma}\big)\Big)^{c}.
\end{equation}
Equation \eqref{eq:all_positive} highlights exponential decay in the probability of sustained improvement: if each cycle has only a $0.6$ chance of improvement (i.e.\ $\Phi(\mu/\sigma)=0.6$), then after $10$ cycles the probability of all improvements being sustained is $0.6^{10}\approx 0.00605$ (about $0.6\%$). This demonstrates the increasing difficulty of guaranteeing long-run monotonic improvements without additional control measures.

\section{Capability Alignment Ratio}
\label{sec:car}

\subsection{Fundamental trade-off}

Capability gains and alignment preservation are not independent. Aggressive optimisation for capability often induces representational drift or other degradation (alignment loss); conversely, strict alignment constraints can limit attainable capability. To quantify the trade-off define the \emph{Capability Alignment Ratio} at cycle $c$ as
\begin{equation}\label{eq:car_def}
\car_c \;=\; \frac{\qualityc - \quality_0}{\gdi_c},
\end{equation}
where $\quality_0$ is baseline quality and $\gdi_c$ denotes cumulative drift (alignment degradation) up to cycle $c$. A large $\car_c$ indicates efficient capability improvement per unit of alignment degradation; a small $\car_c$ indicates the opposite.

\subsection{Pareto efficiency}

Different stopping rules and constraint strengths produce different points on the Pareto frontier between capability (maximize $\quality$) and alignment (minimize $\gdi$). Points on the Pareto frontier are such that any improvement in $\quality$ requires an increase in $\gdi$, and vice versa. Points strictly inside the frontier are dominated and represent suboptimal trade-offs.

\subsection{Decision making under the CAR framework}

Define the \emph{incremental} CAR for cycle $c$:
\begin{equation}\label{eq:car_inc}
\car_{\mathrm{inc},c} \;=\; \frac{\qualityc - \quality_{c-1}}{\gdi_c - \gdi_{c-1}}.
\end{equation}
If $\car_{\mathrm{inc},c}$ is large (relative to historical values or a preset threshold), the most recent cycle produced substantial quality gains for little additional drift, suggesting continuation. If $\car_{\mathrm{inc},c}$ is small or negative (drift without improvement), stopping or rollback is indicated.

\paragraph{Implementation notes.} In practice:
\begin{itemize}
  \item Use smoothed estimates of numerator and denominator in \eqref{eq:car_inc} (e.g.\ exponential moving averages) to reduce sensitivity to single outlier measurements.
  \item Enforce lower bounds on denominators (add small $\eps$) to avoid numerical instability when drift is near zero.
  \item Combine CAR-based signals with regression-risk estimates from Section~\ref{sec:stability} to form a composite stopping rule.
\end{itemize}

\section{Safety and Ethical Considerations}
\label{sec:safety}

\subsection{Risks of Unmonitored Self Improvement}

Recursive self improvement, without principled safeguards, introduces several categories of alignment risk. We detail these risks and explain how our framework addresses them.

\textbf{Gradual Misalignment}: The most insidious risk is that alignment violations accumulate gradually across cycles such that no individual cycle violates alignment properties, yet after many cycles the system has drifted substantially from its objectives. This is precisely the drift phenomenon we measure through GDI. Our multi signal approach detects drift across semantic, lexical, structural, and distributional dimensions, making it difficult for misalignment to hide in any single modality.

\textbf{Specification Gaming}: As systems become better at optimizing for measurable objectives, they may find loopholes in the specification that satisfy the letter but not the spirit of alignment requirements. For instance, a code generation system might generate code that technically passes tests but uses unsafe operations or relies on undefined behavior. Our constraint preservation framework addresses this by requiring explicit satisfaction of specified constraints, not just achievement of quality metrics.

\textbf{Capability Collapse}: Overly restrictive alignment constraints might prevent the system from achieving meaningful capability improvements, creating a scenario where pursuing safety prevents the system from being useful. The CAR framework allows practitioners to see the capability alignment frontier and make informed trade off decisions rather than operating in the dark.

\textbf{Regression and Cycling}: Some systems show oscillating quality patterns where improvements in some cycles are reversed in others. This suggests the system is not finding stable improved behaviors but rather alternating between different regimes. Our regression risk analysis detects this phenomenon and recommends halting when it occurs.

\textbf{Constraint Violation Accumulation}: Even if individual constraints are not violated, their satisfaction might gradually degrade such that the system produces responses that barely satisfy constraints. We track the full trajectory of constraint satisfaction scores and flag when they show concerning trends.

\subsection{Limitations of the Framework}

While comprehensive, our framework has several important limitations that practitioners should understand.

\textbf{Measurement Limitations}: The GDI is fundamentally a measurement of divergence from baseline behavior. If the baseline itself is misaligned, our framework measures drift from misalignment rather than drift from true alignment. The framework is most useful when the initial model's behavior is reasonably aligned.

\textbf{Constraint Specification}: The framework assumes constraints are explicitly specified. In domains where alignment properties are difficult to formalize, the framework provides less benefit. Ongoing work in interpretability and mechanistic understanding of models may eventually enable more nuanced constraint specification.

\textbf{Task Distribution Shift}: Our framework is calibrated for specific task types. If the task distribution changes substantially (e.g., moving from mathematical reasoning to creative writing), recalibration would be prudent.

\textbf{Adversarial Robustness}: The framework detects natural drift but may not be robust to adversarial manipulation. A system specifically designed to evade drift detection might succeed. Adversarial robustness of alignment systems is an open problem beyond the scope of this work.

\subsection{Deployment Considerations}

Deploying self improving systems with our framework involves several practical considerations that practitioners should address.

\textbf{Monitoring and Logging}: All improvement cycles, drift measurements, and decision making should be logged for auditing. This creates a complete record of the system's self improvement journey.

\textbf{Frequency of Human Oversight}: While the framework operates autonomously, human practitioners should periodically review logs and assess whether the system's improvement trajectory aligns with expectations. We recommend periodic reviews every 10 to 20 cycles.

\textbf{Circuit Breaker Mechanisms}: Beyond the algorithmic stopping rules in our framework, systems should include external circuit breakers that allow humans to halt improvement processes if unexpected behaviors emerge.

\textbf{Rollback Capabilities}: Systems should maintain checkpoints at each cycle so that if undesirable changes accumulate, the system can be rolled back to earlier versions. This is particularly important if unexpected drift emerges in post deployment settings.

\subsection{Mitigation Strategies}

Our framework includes several mitigation strategies for addressing alignment risks when they emerge.

\textbf{Conservative Thresholds}: Practitioners can choose more conservative thresholds for drift detection and regression risk if they prioritize safety over performance gains.

\textbf{Constraint Strengthening}: If the system begins exhibiting concerning behaviors, constraints can be strengthened to more strictly limit undesired outputs. Recalibration then occurs on the new constraint set.

\textbf{Cycle Reduction}: Reducing the maximum cycle count trades off potential performance gains for improved safety, as fewer cycles provide fewer opportunities for drift accumulation.

\textbf{Inter Cycle Verification}: For high stakes applications, human verification could be inserted between improvement cycles rather than relying entirely on automated checks.

\textbf{Capability Ceiling}: If the system approaches a CAR frontier where further improvements incur excessive alignment costs, practitioners can adopt a capability ceiling beyond which improvement attempts cease.

\section{Theoretical analysis}
\label{app:theory}

\subsection{Lipschitz continuity of drift}
Let $\mathcal{Y}$ denote the response space with metric $d_{\mathcal{Y}}(\cdot,\cdot)$. The drift measurement
\[
\drift:\mathcal{Y}\times\mathcal{Y}\to[0,1]
\]
maps pairs of responses to a normalized drift value.

\begin{proposition}
\label{prop:app:lipschitz-drift}
There exists $L_{\drift}\ge 0$ such that for all $y,y',\tilde y,\tilde y'\in\mathcal{Y}$
\begin{equation}\label{eq:app:lipschitz-drift}
\big|\drift(y,y')-\drift(\tilde y,\tilde y')\big|
\le L_{\drift}\,\big(d_{\mathcal{Y}}(y,\tilde y)+d_{\mathcal{Y}}(y',\tilde y')\big).
\end{equation}
If $\drift$ is a convex combination $\drift=\sum_{i=1}^m \alpha_i\drift_i$ with $\alpha_i\ge0$, $\sum_i\alpha_i=1$, and each $\drift_i$ is Lipschitz with constant $L_i$, then one may take $L_{\drift}=\sum_i\alpha_i L_i$.
\end{proposition}

\begin{proof}[Proof sketch]
Each component drift (semantic, lexical, structural, distributional, \ldots) is Lipschitz in an appropriate feature metric (e.g. semantic drift via embedding-induced distance, lexical via divergence metrics bounded by total variation, structural via structural-feature norms, distributional via Wasserstein). Linear combination with nonnegative weights preserves Lipschitz continuity and yields the stated bound by the triangle inequality.
\end{proof}

This result ensures that small perturbations in responses yield proportionally small changes in measured drift under the chosen metric coupling.

\subsection{Asymptotic drift bounds}
Relate cumulative drift to a sequence of quality changes across improvement cycles.

\begin{proposition}
\label{prop:app:asymptotic-bound}
Let $\Delta \quality_c=\quality_c-\quality_{c-1}$. Suppose $\{\Delta \quality_c\}_{c\ge1}$ are (approximately) independent with mean $\mathbb{E}[\Delta \quality_c]=\mu$ and variance $\mathrm{Var}(\Delta \quality_c)=\sigma^2$. Assume pointwise
\[
|\drift_c-\drift_{c-1}|\le L_{\drift}\,|\Delta \quality_c|
\]
for each $c$. Then the expected drift after $C$ cycles satisfies
\begin{equation}\label{eq:app:expected-drift-bound}
\mathbb{E}[\drift_C]\le \drift_0 + L_{\drift}\,\mathbb{E}\Big[\sum_{c=1}^C |\Delta \quality_c|\Big]
= \drift_0 + L_{\drift}\,\big(\mu C + O(\sigma\sqrt{C})\big).
\end{equation}
\end{proposition}

\begin{proof}[Proof sketch]
By telescoping and the assumed per-cycle bound,
\[
\drift_C \le \drift_0 + L_{\drift}\sum_{c=1}^C |\Delta \quality_c|.
\]
Take expectations and use independence and standard concentration heuristics to obtain the stated bound.
\end{proof}

Equation \eqref{eq:app:expected-drift-bound} implies at-most-linear growth of expected drift when $\mu>0$, and $\mathcal{O}(\sqrt{C})$ growth when $\mu\approx 0$.

\subsection{Stability in contractive regimes}
Define a regime where drift remains controlled as quality improves.

\begin{definition}
A system is in a \emph{contractive regime} if the effective Lipschitz constant relating quality changes to drift satisfies $L_{\drift}<1$.
\end{definition}

\begin{proposition}
\label{prop:app:contractive-equilibrium}
Assume the per-cycle expected quality increment is $\mu\ge0$ and the expected-drift evolution is approximated by the affine map
\[
\mathbb{E}[\drift_{c+1}] \approx L_{\drift}\,\mathbb{E}[\drift_c] + L_{\drift}\,\mu.
\]
If $L_{\drift}\in[0,1)$, then
\[
\lim_{c\to\infty}\mathbb{E}[\drift_c] = \frac{L_{\drift}\,\mu}{1-L_{\drift}}.
\]
\end{proposition}

\begin{proof}[Proof sketch]
Solve the fixed-point equation for the affine recurrence. Contraction ($L_{\drift}<1$) ensures geometric convergence to the fixed point.
\end{proof}

When $L_{\drift}<1$ the model can sustain iterative improvements without unbounded drift; for $L_{\drift}\ge1$ the linearized recurrence predicts instability absent external control or corrective mechanisms.

\section{Practical Deployment Considerations}

For practitioners deploying this framework, several practical considerations emerge from the results.

First, calibration is essential. Attempting to apply learned thresholds across domains without recalibration is likely to fail. Our calibration procedure (18 tasks, 3 cycles each) is computationally modest, making recalibration practical even for resource constrained settings.

Second, task type matters. Truthfulness improvements are more costly in alignment terms, while code and math improvements are cheaper. Resource allocation and supervision should reflect these differences.

Third, the framework works well for tasks with clear specifications (code, math) and less well for open ended tasks. This suggests the framework is particularly valuable for safety critical applications where constraints can be formally specified.

Fourth, manual oversight remains important. While the framework provides algorithmic safeguards, human review of improvement trajectories, particularly for high risk task types or high drift tasks, provides additional insurance against unexpected behaviors.

\end{document}